\documentclass[]{ceurart}

\usepackage{algorithm}
\usepackage{algpseudocode}
\usepackage{tabularx}
\usepackage{booktabs}
\usepackage{multirow}
\usepackage{todonotes}
\usepackage[scr=boondoxo]{mathalpha}
\usepackage{amsthm}
\usepackage{hyperref}
\newtheorem{definition}{Definition}
\newtheorem{proposition}{Proposition}
\newcommand{\email}[1]{\texttt{\rm\texttt{#1}}}

\begin{document}

\copyrightyear{2024}
\copyrightclause{Copyright for this paper by its authors.
  Use permitted under Creative Commons License Attribution 4.0 International (CC BY 4.0).}

\conference{KBC-LM'24: Knowledge Base Construction from Pre-trained Language Models workshop at ISWC 2024}

\title{Enriching Ontologies with Disjointness Axioms using Large Language Models}

\author[1]{Elias Crum}\cormark[1]

\author[2]{Antonio {De Santis}}\cormark[1]

\author[3]{Manon Ovide}\cormark[1]

\author[4]{Jiaxin Pan}\cormark[1]

\author[5]{Alessia Pisu}\cormark[1]

\author[6]{Nicolas Lazzari}
\author[7]{Sebastian Rudolph}
%
%
%

\address[1]{Ghent University, Belgium\quad \email{elias.crum@ugent.be}} 
\address[2]{Politecnico di Milano, Italy\quad \email{antonio.desantis@polimi.it}}
\address[3]{University of Tours, France\quad \email{manon.ovide@univ-tours.fr}}
\address[4]{University of Stuttgart, Germany\quad \email{jiaxin.pan@ki.uni-stuttgart.de}}
\address[5]{University of Cagliari, Italy\quad \email{alessia.pisu96@unica.it}}
\address[6]{University of Pisa and University of Bologna, Italy\quad \email{nicolas.lazzari3@unibo.it}}
\address[7]{TU Dresden, Germany\quad \email{sebastian.rudolph@tu-dresden.de}}

\cortext[1]{These authors contributed equally.}
\tnotetext[0]{This paper presents joint research that originated from the team project of ``House Slytherin'' at the 2024 International Semantic Web Summer School (ISWS) in Bertinoro, Italy.}

\begin{abstract}
Ontologies often lack explicit disjointness declarations between classes, despite their usefulness for sophisticated reasoning and consistency checking in Knowledge Graphs. 
In this study, we explore the potential of Large Language Models (LLMs) to enrich ontologies by identifying and asserting class disjointness axioms.
Our approach aims at leveraging the implicit knowledge embedded in LLMs, using prompt engineering to elicit this knowledge for classifying ontological disjointness. We validate our methodology on the DBpedia ontology, focusing on open-source LLMs. Our findings suggest that LLMs, when guided by effective prompt strategies, can reliably identify disjoint class relationships, thus streamlining the process of ontology completion without extensive manual input. For comprehensive disjointness enrichment, we propose a process that takes logical relationships between disjointness and subclass statements into account in order to maintain satisfiability and reduce the number of calls to the LLM. 
This work provides a foundation for future applications of LLMs in automated ontology enhancement and offers insights into optimizing LLM performance through strategic prompt design. Our code is publicly available on GitHub at \url{https://github.com/n28div/llm-disjointness}.
\end{abstract}

\maketitle

\section{Introduction}
\label{sec:intro}

It is generally understood that complementing the factual (assertional) knowledge represented in Knowledge Graphs with ontological (terminological) information greatly advances the usefulness of the ensuing knowledge base in terms of querying and many other downstream tasks. This is because combining assertional information with terminological background knowledge allows for the derivation of a vast amount of implicit knowledge, which is not explicitly stated in the knowledge base but follows logically from it and thus can be taken into account for all kinds of knowledge management activities, including query answering.

The by far most widespread type of ontological information added to knowledge graphs is \emph{taxonomic} in nature, that is, it is related to (i) putting the individual objects of interest into categories, usually referred to as \emph{classes}, based on shared characteristics and (ii) establishing set-theoretic relationships between these classes. Among the diverse possible such taxonomic relationships, the \emph{subclass/superclass} relationships -- tightly connected to the linguistic hyponymy/hypernymy relationships of the corresponding class names -- are the ones predominantly found across numerous ontologies today, typically forming sizeable conceptual hierarchies. As an example, the subclass/superclass relationship between the classes \texttt{Mammal} and \texttt{Vertebrate} implies that any object that belongs to (or, in more technical terms: is an \emph{instance} of) the class \texttt{Mammal} also must belong to the class \texttt{Vertebrate}.

Another well-known basic type of taxonomic relationship between two classes is that of \emph{disjointness}. Two classes are said to be \emph{disjoint} if it is impossible that they have common instances, which, intuitively, means that the two classes cannot overlap, and membership in these two classes is mutually exclusive. For example, disjointness of the classes 
\texttt{Mammal} and \texttt{Fish} implies that any instance of \texttt{Mammal} must not be an instance of \texttt{Fish}. Given the symmetric nature of disjointness, this is logically equivalent to saying that any instance of \texttt{Fish} must not be an instance of \texttt{Mammal}. As opposed to subclass statements, which allow for inferring positive facts from other positive facts, disjointness statements enable the inference of \emph{negated} facts. For example, given the fact that Flipper is an instance of \texttt{Mammal}, the above subclass relationship gives rise to the information that Flipper is an instance of \texttt{Vertebrate}, whereas the disjointness statement allows us to infer the information that Flipper must not be an instance of \texttt{Fish}. This fact makes disjointness information particularly valuable in the context of machine-learning approaches that rely on the presence of negative examples, such as Knowledge Graph Embedding.

When specifying taxonomic relationships between classes in the course of the ontology design process, it should be kept in mind that they are not meant to reflect spurious relationships in the data currently available, but rather they are supposed to represent immutable background knowledge that continues to hold in different situations or at different points in time. For instance, although historically, no woman has served as US President, a woman may be elected as the US President in the future. Therefore, the corresponding classes \texttt{Woman} and \texttt{USPresident} are not (ontologically) disjoint.\footnote{We might call them \textit{materially disjoint} due to the absence of material evidence demonstrating their non-disjointness.}
To reflect this situation more formally, one can employ the idea of \emph{possible} or \emph{conceivable worlds} (referred to as \emph{interpretations} in model-theoretic terms), which e.g., include potential future or just hypothetical circumstances. Then, a certain taxonomic relationship (such as subclass or disjointness) between two classes holds if the corresponding set relation (such as subset or intersection-emptiness) holds between the sets of class instances \emph{in every conceivable world} (\emph{under every conceivable interpretation}). 
Based on this, we will employ a very lightweight logical framework to give our arguments a formal underpinning: Stipulating a set $\mathscr{I}$ of conceivable worlds, we define taxonomic relationships for this set. The goal of ontological knowledge modeling is to capture $\mathscr{I}$ using a knowledge base $\mathcal{K}$ whose statements rule out the inconceivable worlds so that only the conceivable ones remain as models of $\mathcal{K}$.

\begin{definition}\label{def:taxo}
Fixing a vocabulary consisting of a set $\mathbf{C}$ of class names and a set $\mathbf{I}$ of individual names, an \emph{interpretation} $\mathcal{I} = (\Delta,\cdot^\mathcal{I})$ consists of a set $\Delta$ called the \emph{domain} and a function $\cdot^\mathcal{I}$ mapping every class name $\mathtt{C} \in \mathbf{C}$ to a subset $\mathtt{C}^\mathcal{I} \subseteq \Delta$ and every individual name $\mathtt{i} \in \mathbf{I}$ to an element $\mathtt{i}^\mathcal{I} \in \Delta$.

Let $\mathscr{I}$ be a set of interpretations, representing the \emph{conceivable worlds}.
Then, for an individual name $\mathtt{i} \in \mathbf{I}$ and for concept names $\mathtt{C}, \mathtt{D} \in \mathbf{C}$ we call
\begin{itemize}
\item 
$\mathtt{i}$ an \emph{instance of} $\mathtt{C}$ (written $\mathtt{i}:\mathtt{C}$) if every interpretation $\mathcal{I} \in \mathscr{I}$ satisfies $\mathtt{i}^\mathcal{I} \in \mathtt{C}^\mathcal{I}$,
\item 
$\mathtt{C}$ a \emph{subclass of} $\mathtt{D}$ (written $\mathtt{C} \sqsubseteq \mathtt{D}$) if every interpretation $\mathcal{I} \in \mathscr{I}$ satisfies $\mathtt{C}^\mathcal{I} \subseteq \mathtt{D}^\mathcal{I}$,
\item 
$\mathtt{C}$ \emph{disjoint with} $\mathtt{D}$ if every interpretation $\mathcal{I} \in \mathscr{I}$ satisfies $\mathtt{C}^\mathcal{I} \cap \mathtt{D}^\mathcal{I} = \emptyset$.
\item 
$\mathtt{C}$ \emph{incoherent} if every interpretation $\mathcal{I} \in \mathscr{I}$ satisfies $\mathtt{C}^\mathcal{I} = \emptyset$.
\end{itemize}

\end{definition}




\noindent As discussed above, ontologically dictated taxonomic relationships can be leveraged for sophisticated reasoning and consistency-checking tasks when reasoning over a knowledge graph. Yet, despite their usefulness, disjointness relationships are rarely explicitly recorded within an ontology. 
Research on 1,275 ontologies showed that only 97 of them include disjointness assertions \cite{wang2006gauging}. 
Arguably, this can be explained by the fact that disjointness information is so self-evident from a human common-sense point of view, that human experts are often not aware that it is not logically ``built-in'' but needs to be explicitly specified.  
For this reason, semi-automated labeling of disjoint classes could be advantageous. Recent approaches \cite{volker2007learning,volker2015automatic,rizzo2021unsupervised} propose supervised and unsupervised models using various features in disjointness axioms. However, the generalizability of these methods is limited to their specific datasets and cannot be implemented on a large scale. Additionally, the sophisticated feature engineering required hinders their practical application. Therefore, a method that functions independently of feature design and dataset restrictions is highly desirable.

Given that (i) ontological class descriptions are often recorded as (or associated with) terms in natural language and (ii) LLMs have been found to possess wide linguistic and semantic working knowledge, we aim to assess the potential of LLMs to decide on the question which classes ought to be disjoint while assessing the impact of prompt engineering on classification validity. We hypothesize that through the use of prompt engineering, LLMs are to classify ontologically disjoint classes with high validity in both positive (two classes are ontologically disjoint), and negative (two classes are not ontologically disjoint), cases. We test our hypothesis on the DBpedia ontology\footnote{https://DBpedia.org/ontology/} using LLMs. We propose a method that intertwines the LLM-based disjointness classification with basic logical inferencing to increase efficiency, maintain consistency, and minimize the number of calls to the LLM.


Thus, this paper is dedicated to answering the following main research questions:
\begin{itemize}
    \item[] RQ1: Can LLMs help enrich ontologies with class disjointness axioms?
    \item[] RQ2: Which LLM prompts work better for disjointness discovery?
    \item[] RQ3: How can we exploit taxonomic relationships to reduce interaction with the LLM?
\end{itemize}



\section{Related Work}
\label{sec:related}

\paragraph{Disjointness Learning}

Models for disjointness learning can be categorized into supervised and unsupervised approaches. In the unsupervised category, \citet{schlobach2005debugging} follows the \emph{strong disjointness assumption} \cite{cornet2002usability}, which posits that children of a common parent in the subsumption hierarchy should be considered disjoint. They introduced a pinpointing algorithm to identify minimal sets of axioms that need revision to make an ontology coherent, thereby enriching appropriate disjointness statements. However, this approach neglects background knowledge, which could be beneficial in identifying disjoint classes. \citet{rizzo2021unsupervised} proposes an unsupervised approach based on concept learning and inductive classification. This method employs a hierarchical conceptual clustering technique capable of providing intensional cluster descriptions and utilizes a novel form of semi-distances over individuals in an ontological knowledge base, incorporating available background knowledge.
In the supervised category, \citet{volker2007learning,volker2015automatic} gather syntactic and semantic evidence, such as positive and negative association rules as well as correlation coefficients, from various sources to establish a strong foundation for learning disjointness.
However, their work exploits background knowledge and reasoning only to a limited extent. Subsequent work, the DL-Learner by \citet{lehmann2009dl}, uses Inductive Logic Programming (ILP) for learning class descriptions, including disjointness.
Despite these advancements, disjointness learning with LLMs remains much underexplored.

\paragraph{Large Language Models}
In recent years, Large Language Models (LLMs) have become state-of-the-art for Natural Language Processing and have also significantly impacted other fields such as knowledge engineering \cite{allen,roadmap,integratinglargelanguagemodels,petroni-etal-2019-language}. LLMs rely on pre-training Transformer models \cite{vaswani2017attention} over large-scale unlabeled corpora. 
Pre-trained context-aware word representations achieve state-of-the-art performance on various downstream tasks and set the ``pre-training and fine-tuning'' learning paradigm. 
Early LLMs, such as BERT \cite{kenton2019bert}, utilized relatively small training corpora and required fine-tuning for specific downstream tasks. However, subsequent research demonstrated that scaling up both model size and dataset volume significantly enhances performance. GPT-3 \cite{brown2020language}, for instance, achieves competitive results through few-shot learning and in-context learning without parameter updates. GPT-3.5 further improves capabilities by incorporating reinforcement learning from human feedback (RLHF).
The introduction of GPT-4 \cite{openai2023gpt} marked a milestone by extending beyond text input to include multimodal signals. Meta AI introduced the collection of LLaMA models \cite{touvron2023llama,touvron2023llama2} with four different sizes. Other notable LLMs, such as Claude, Gemini \cite{reid2024gemini}, and Mixtral \cite{jiang2024mixtral}, have also garnered significant attention.

\paragraph{Prompt Engineering}
Designing effective prompts for LLMs is essential for maximizing their potential. 
Key strategies in prompt engineering include zero-shot \cite{radford4language}, few-shot \cite{brown2020language}, and chain-of-thought \cite{chen2024unleashing} prompting. \emph{Zero-shot} \cite{radford4language} involves providing task descriptions to LLMs without any input-output examples, relying on the models’ pre-existing knowledge to generate responses. \emph{Few-shot} \cite{brown2020language} includes input-output examples, guiding the models’ generation process. \emph{Chain-of-Thought (CoT)} \cite{wei2022chainofthought} promotes coherent and step-by-step reasoning by decomposing a complex question into a series of simpler logical reasoning questions,  mimicking human problem-solving processes. This method has been shown to significantly improve performance on reasoning tasks \cite{wei2022chainofthought}.
However, the need for multiple prompts makes this approach difficult to use at large scales. With this in mind, \citet{kojima2023large} proposed \emph{Zero-shot-CoT} prompting. They found that by appending the phrase ``Let's think step by step.'' to the end of a question, LLMs can generate a chain of thought that leads to more accurate answers w.r.t the vanilla zero-shot approach.

\section{Resources} \label{sec:resources}

To effectively assess the ability of LLMs to support the assertion of disjointness axioms, we ideally require a reference ontology that includes a sized set of classes, to ensure diversity during the experiments and some disjoint classes in its description, preferably specified through a specific disjoint class property such as \texttt{owl:disjointWith}.
These criteria maximize the generalizability of the approach and encourage its use for future studies. 

Several ontologies can be identified for this task, from foundational ontologies, such as DOLCE\footnote{\url{https://github.com/appliedontolab/DOLCE/blob/main/OWL/DOLCEbasic.owl} \\ \url{http://www.ontologydesignpatterns.org/ont/dul/DUL.owl}} or UFO\footnote{\url{https://nemo-ufes.github.io/gufo/}}, to domain-specific ontologies, such as FoodOn\footnote{\url{https://foodon.org/}}. Disjointness axioms from these ontologies, however, are not intuitive and require extensive common-sense reasoning and domain knowledge. For instance, DOLCE defines an \texttt{Event} to be disjoint from an \texttt{Object} while UFO does not. Both axioms are correct, as they deeply depend on their philosophical commitment to these abstract concepts. Similarly, the FoodOn ontology asserts that the \textit{Arabia coffee plant} \footnote{\url{https://en.wikipedia.org/wiki/Coffea_arabica}}, the plant used to produce black coffee, is disjoint with \textit{Camellia sinensis} \footnote{\url{https://en.wikipedia.org/wiki/Camellia_sinensis}}, the plant used to produce black tea. In this case, deciding whether the two plants should be considered disjoint highly depends on the domain of the ontology.
To avoid feeding the LLM with classes whose disjointness highly depends on the context or domain, we choose to avoid foundational and domain-specific ontologies for our initial experiments. Moreover, as our interaction with the LLM is based on natural language, we only consider ontologies that provide natural language labels for classes via labeling properties, such as \texttt{skos:prefLabel} or \texttt{rdfs:label}.

We ultimately decided to use the DBpedia ontology\footnote{\url{https://DBpedia.org/ontology/}, often referred to with the \texttt{dbo:} namespace, which we omit hereafter} because of its general popularity and conformity with dataset minimal requirements. Since the DBpedia ontology is created through a crowdsourcing approach \cite{lehmann2015DBpedia}, the availability of disjointness axioms cannot be expected to be equally accurate across all classes, as it depends on the annotators' expertise and diligence. This issue has been actively discussed within the DBpedia community\footnote{\url{https://github.com/DBpedia/ontology-tracker/issues/2}}. The main drawback is the lack of a systematic approach in the creation of the taxonomy, which greatly impacts the consistency of the ontology when disjointness axioms are asserted. 
\begin{table}[ht]
    \centering
    \begin{tabular}{ll}
    \toprule
    Class A & Class B \\ \midrule
    \url{http://DBpedia.org/ontology/Person} & \url{http://DBpedia.org/ontology/ProtohistoricalPeriod} \\
    \url{http://DBpedia.org/ontology/Person} & \url{http://DBpedia.org/ontology/UnitOfWork} \\
    \url{http://DBpedia.org/ontology/Agent} & \url{http://DBpedia.org/ontology/Place} \\
    \url{http://DBpedia.org/ontology/Fish} & \url{http://DBpedia.org/ontology/Mammal} \\
    \url{http://DBpedia.org/ontology/Event} & \url{http://DBpedia.org/ontology/Person} \\
    \bottomrule
    \end{tabular}
    \caption{Examples of pairs of classes explicitly specified as ontologically disjoint in the DBpedia ontology.}
    \label{tab:DBpedia-example}
\end{table}
In particular, we found $23$ explicit disjointness axioms in the DBpedia ontology. In Section \ref{sec:approach} we show how exploiting automated reasoning techniques allows the creation of a larger pool of disjoint classes. In Table \ref{tab:DBpedia-example} a selection of disjointness axioms within the ontology is shown. Indeed, most of the disjointness axioms are \textit{universally} known common-sense relations, such as disjointness between \texttt{dbo:Fish} and \texttt{dbo:Mammal} or \texttt{dbo:Agent} and \texttt{dbo:Place}.

\newcommand{\K}{\mathcal{K}}
\newcommand{\D}{\mathcal{D}}
\section{Proposed approach} \label{sec:approach}
We now describe our approach 
which, given a Knowledge Base, clarifies for \emph{every} pair of named classes of that ontology if disjointness should hold between the two classes or not. 
At the core of the approach is prompting an LLM to exploit the semantic and linguistic ``world knowledge'' it has obtained from training on vast amounts of textual data. 
The two major underlying objectives of our approach are:

\vspace{-0.47ex}

\begin{enumerate}
\item Ensuring that the resulting disjointness-enriched ontology is satisfiable (i.e., contradiction-free) for usability reasons since otherwise it would be unusable for any reasoning tasks, including ontology-supported querying.
\item Minimizing the number of interactions with the LLM for efficiency reasons and cost-awareness.
\end{enumerate}

\vspace{-0.47ex}

\noindent We propose to address both objectives using automated reasoning. More specifically, we continuously materialize all the (non-)disjointness information that follows logically from the original knowledge base plus the already acquired disjointness information. Thus, the LLM is only queried about the disjointness status of pairs of classes, when neither of the outcomes would result in an inconsistency. In this way, the derived information remains contradiction-free ``by design'' and, at the same time, the number of queries to the LLM is significantly reduced.    
Our approach relies on several logical correspondences, discussed in the following.


\begin{proposition} \label{prop:common-subclasses}
Let $\K$ be a knowledge base and let  $\mathtt{C}_1$, $\mathtt{C}_2$, $\mathtt{D}_1$, $\mathtt{D}_2$ be classes of $\K$ such that the following statements follow from $\K$:
(i) $\mathtt{C}_1$ and $\mathtt{C}_2$ are disjoint,
(ii) $\mathtt{D}_1$ is a subclass of $\mathtt{C}_1$,
(iii) $\mathtt{D}_2$ is a subclass of $\mathtt{C}_2$.
Then $\K$ also entails that $\mathtt{D}_1$ and $\mathtt{D}_2$ are disjoint.
\end{proposition}

\begin{proof}
Consider an arbitrary model $\mathcal{I}$ of $\K$. According to the assumptions and in view of Definition~\ref{def:taxo}, we know that (i)~$\mathtt{C}_1^\mathcal{I} \cap \mathtt{C}_2^\mathcal{I} = \emptyset$,  (ii)~$\mathtt{D}_1^\mathcal{I} \subseteq \mathtt{C}_1^\mathcal{I}$, and  (iii)~$\mathtt{D}_2^\mathcal{I} \subseteq \mathtt{C}_2^\mathcal{I}$. We equivalently express (ii) and (iii) by (ii')~$\mathtt{D}_1^\mathcal{I} = \mathtt{C}_1^\mathcal{I} \cap \mathtt{D}_1^\mathcal{I}$, and  (iii')~$\mathtt{D}_2^\mathcal{I} = \mathtt{C}_2^\mathcal{I} \cap \mathtt{D}_2^\mathcal{I}$. This allows us to infer $\mathtt{D}_1^\mathcal{I} \cap \mathtt{D}_2^\mathcal{I} = (\mathtt{C}_1^\mathcal{I} \cap \mathtt{D}_1^\mathcal{I}) \cap (\mathtt{C}_2^\mathcal{I} \cap \mathtt{D}_2^\mathcal{I}) = (\mathtt{C}_1^\mathcal{I} \cap \mathtt{C}_2^\mathcal{I}) \cap (\mathtt{D}_1^\mathcal{I} \cap \mathtt{D}_2^\mathcal{I}) = \emptyset \cap (\mathtt{D}_1^\mathcal{I} \cap \mathtt{D}_2^\mathcal{I}) = \emptyset$. 
\end{proof}

\noindent We exploit this property to use subclass relationships from $\K$ to deduce class disjointness statements from existing class disjointness statements. This way we avoid posing redundant disjointness queries to the underlying LLM.

\begin{proposition} \label{prop:sub-disjoint}
Let $\K$ be a knowledge base and let  $\mathtt{C}_1$, $\mathtt{C}_2$, $\mathtt{C}$ be classes of $\K$ such that the following statements follow from $\K$:
(i) $\mathtt{C}_1$ and $\mathtt{C}_2$ are disjoint,
(ii) $\mathtt{C}$ is a subclass of $\mathtt{C}_1$,
(iii) $\mathtt{C}$ is a subclass of $\mathtt{C}_2$.
Then $\K$ also entails that $\mathtt{C}$ is incoherent.
\end{proposition}

\begin{proof}
Consider an arbitrary model $\mathcal{I}$ of $\K$. According to the assumptions and in view of Definition~\ref{def:taxo}, we know that (i)~$\mathtt{C}_1^\mathcal{I} \cap \mathtt{C}_2^\mathcal{I} = \emptyset$,  (ii)~$\mathtt{C}^\mathcal{I} \subseteq \mathtt{C}_1^\mathcal{I}$, and  (iii)~$\mathtt{C}^\mathcal{I} \subseteq \mathtt{C}_2^\mathcal{I}$. We equivalently express (ii) and (iii) by (ii')~$\mathtt{C}^\mathcal{I} = \mathtt{C}_1^\mathcal{I} \cap \mathtt{C}^\mathcal{I}$, and  (iii')~$\mathtt{C}^\mathcal{I} = \mathtt{C}_2^\mathcal{I} \cap \mathtt{C}^\mathcal{I}$. This allows us to infer $\mathtt{C}^\mathcal{I} = \mathtt{C}^\mathcal{I} \cap \mathtt{C}^\mathcal{I} = (\mathtt{C}_1^\mathcal{I} \cap \mathtt{C}^\mathcal{I}) \cap (\mathtt{C}_2^\mathcal{I} \cap \mathtt{C}^\mathcal{I}) = (\mathtt{C}_1^\mathcal{I} \cap \mathtt{C}_2^\mathcal{I}) \cap \mathtt{C}^\mathcal{I} = \emptyset  \cap \mathtt{C}^\mathcal{I} = \emptyset$. 
\end{proof}

\noindent We exploited this property indirectly under the assumption that any named class $C$ in the considered ontology is supposed to have instances -- which seems to be a reasonable assumption since, otherwise, the definition of the class appears to be meaningless. In that case, any two classes that have a common subclass must be not disjoint. 

\begin{proposition} \label{prop:material-disj}
Let $K$ be a knowledge base, let  $\mathtt{C}_1$, $\mathtt{C}_2$ be classes and let $e$ be an individual of $\K$ that such that the following statements follow from $\K$:
(i) $\mathtt{C}_1$ and $\mathtt{C}_2$ are disjoint,
(ii) $e$ is an instance of $\mathtt{C}_1$,
(iii) $\mathtt{e}$ is an instance of $\mathtt{C}_2$.
Then $\K$ is unsatisfiable.
\end{proposition}

\begin{proof}
Suppose $\mathcal{I}$ is a model of $\K$. According to the assumptions and in view of Definition~\ref{def:taxo}, we know that (i)~$\mathtt{C}_1^\mathcal{I} \cap \mathtt{C}_2^\mathcal{I} = \emptyset$,  (ii)~$\mathtt{e}^\mathcal{I} \in \mathtt{C}_1^\mathcal{I}$, and  (iii)~$\mathtt{e}^\mathcal{I} \in \mathtt{C}_2^\mathcal{I}$. Then, combining (ii) and (iii) we obtain $\mathtt{e}^\mathcal{I} \in \mathtt{C}_1^\mathcal{I} \cap \mathtt{C}_2^\mathcal{I}$ and applying (i) yields $\mathtt{e}^\mathcal{I} \in \emptyset$ which is a contradictory statement. Thus $\K$ cannot have any models, which means it is unsatisfiable.
\end{proof}

\begin{algorithm}
    \caption{Determine the pair of disjoint classes derivable from $\K$}
    \label{algo:disjoint-pairs}
    \hspace*{\algorithmicindent} \textbf{Input} A knowledge base $\K$ containing a class hierarchy. The knowledge base might or might not contain defined disjointness axioms and/or more complex axiomatizations. \\
    \hspace*{\algorithmicindent} \textbf{Output} A list $\mathcal{L}$ of pair of classes such that disjointness statements logically following from $\K$ are explicitly asserted.
    \begin{algorithmic}[1]
    \State Create a list $\mathcal{L}$ of all pairs of classes $(\mathtt{C}_1, \mathtt{C}_2) \mid \mathtt{C}_1$ is lexicographically smaller than $\mathtt{C}_2$. 
    
    \State Label all entries in $\mathcal{L}$ as ``unknown''.
    
    \ForAll{$\mathtt{D}_1$ \textbf{disjoint with} $\mathtt{D}_2$ \textbf{in} $\K$}
        \ForAll{$(\mathtt{C}_1, \mathtt{C}_2) \in \mathcal{L}$}
            \If{($\mathtt{C}_1 \sqsubseteq \mathtt{D}_1 \land \mathtt{C}_2 \sqsubseteq \mathtt{D}_2$) \textbf{or} ($\mathtt{C}_1 \sqsubseteq \mathtt{D}_2 \land \mathtt{C}_2 \sqsubseteq \mathtt{D}_1$)}
                \State $(\mathtt{C}_1, \mathtt{C}_2) \gets$ ``disjoint''
            \EndIf
        \EndFor
    \EndFor
    
    \ForAll{$(\mathtt{C}_1, \mathtt{C}_2) \in \mathcal{L} \mid (\mathtt{C}_1, \mathtt{C}_2 ) =\ $``unknown''}
        \If{$\mathtt{C}_1$ \textbf{and} $\mathtt{C}_2$ \textbf{have joint subclasses}}
            \State $(\mathtt{C}_1, \mathtt{C}_2) \gets\ $ ``not disjoint''
        \EndIf
    \EndFor
    
    \ForAll{$(\mathtt{C}_1, \mathtt{C}_2) \in \mathcal{L} \mid (\mathtt{C}_1, \mathtt{C}_2) =\ $``unknown''}
        \State Query $K$ for joint instances of $\mathtt{C}_1$ and $\mathtt{C}_2$.
        \If{$\exists \mathtt{e} \in \mathbf{I} \mid \mathtt{e} : \mathtt{C}_1 \land \mathtt{e} : \mathtt{C}_2 $}
            \State $(\mathtt{C}_1, \mathtt{C}_2) \gets\ $ ``not disjoint''
        \EndIf
    \EndFor
    
    \end{algorithmic}
\end{algorithm}
\begin{algorithm}
    \caption{Determine the set of disjointness statements $\D$ consistent with $\K$}
    \label{algo:query-llm}
    \hspace*{\algorithmicindent} \textbf{Input} A list $\mathcal{L}$ containing pairs of classes labeled as ``unknown'', ``disjoint'' or ``not disjoint'';\\ a prompt $P$ for disjointness classification, with 
    $LLM_P: \mathbf{C} \times \mathbf{C} \rightarrow \{ \text{``disjoint''},\! \text{``not disjoint''}\}$ the function that queries an LLM for disjointness of two classes using prompt $P$\\
    \hspace*{\algorithmicindent} \textbf{Output} A set $\D$ of class disjointness axioms, such that all valid disjointness statements logically follow from $\K \cup \D$ and no invalid disjointness statements follow from it.
    
    \begin{algorithmic}[1]    
    \While{$\exists (\mathtt{D}_1, \mathtt{D}_2) \in \mathcal{L} \mid (\mathtt{D}_1, \mathtt{D}_2) =\ $ ``unknown''}
        \State Select $(\mathtt{D}_1, \mathtt{D}_2) \in \mathcal{L} \mid (\mathtt{D}_1, \mathtt{D}_2) =\ $ ``unknown''
        \State $d \gets LLM_P(\mathtt{D}_1, \mathtt{D}_2)$        
        \If{$d =$ ``disjoint''} 
            \ForAll{$(\mathtt{C}_1, \mathtt{C}_2) \in \mathcal{L} \mid (\mathtt{C}_1 \sqsubseteq \mathtt{D}_1 \land \mathtt{C}_2 \sqsubseteq \mathtt{D}_2) \lor (\mathtt{C}_2 \sqsubseteq \mathtt{D}_1 \land \mathtt{C}_1 \sqsubseteq \mathtt{D}_2)$}
                \State $(\mathtt{C}_1, \mathtt{C}_2) \gets$ ``disjoint''
            \EndFor
        \Else
            \ForAll{$(\mathtt{C}_1, \mathtt{C}_2) \in \mathcal{L} \mid (\mathtt{D}_1 \sqsubseteq \mathtt{C}_1 \land \mathtt{D}_2 \sqsubseteq \mathtt{C}_2) \lor (\mathtt{D}_2 \sqsubseteq \mathtt{C}_1 \land \mathtt{D}_1 \sqsubseteq \mathtt{C}_2)$}
                \State $(\mathtt{C}_1, \mathtt{C}_2) \gets$ ``not disjoint''
            \EndFor
        \EndIf
    \EndWhile

    \State $\D^* \gets \{ \texttt{DisjointClasses}(\mathtt{C}_1, \mathtt{C}_2) \mid (\mathtt{C}_1, \mathtt{C}_2) \in \mathcal{L} \land (\mathtt{C}_1, \mathtt{C}_2) = \text{``disjoint''}\}$
    \State Determine the minimal subset $\D$ of $\D^*$ such that $\K \cup \D$ entails $\D^*$
    \end{algorithmic}
\end{algorithm}

\noindent Again, this property can be exploited by noting that any two classes having common instances must not be disjoint. These considerations lead to the proposed methodology, detailed in Algorithm \ref{algo:disjoint-pairs}, which achieves the above-mentioned objective of producing an enriched knowledge base that is guaranteed to be contradiction-free, provided that the original knowledge base is.

Algorithm \ref{algo:query-llm} achieves the objective of reducing the number of interactions with the LLM and maintaining satisfiability as new disjointness information is added through the LLM. The aim of producing an output that accurately reflects taxonomic relationships crucially depends on the quality and accuracy of the LLM's responses. This, in turn, is influenced by both the LLM itself and the chosen prompting strategy. We focus on these issues in Section~\ref{sec:experiments}.

The last steps of Algorithm \ref{algo:query-llm} (lines 14 and 15) are optional, but highly recommended, as they remove logically redundant statements from the disjointness-enriched knowledge base $\K \cup \D$. This yields a knowledge base that is logically equivalent but typically much smaller in size and hence both easier to process algorithmically and to scrutinize and maintain manually.
Also, this ``pruning step'' is not computationally expensive, as it only requires $|\D^*|$ calls to a reasoner.

\section{Experiments} \label{sec:experiments}
In this section, we experiment with the approach proposed in Section \ref{sec:approach} on the classes extracted from the DBpedia ontology. In particular, by relying on Algorithm \ref{algo:disjoint-pairs}, we obtain the list $\mathcal{L}$ related to the DBpedia ontology. We have that $|\mathcal{L}| = 1148$, with $370$ pairs labeled as disjoint and $778$ pairs labeled as not \textit{unknown}. 
In Table \ref{tab:derived-disjoint}, we provide some examples of classes in $\mathcal{L}$.

\begin{table}[ht]
    \caption{Examples of disjointness between classes derived from Propositions \ref{prop:common-subclasses}, \ref{prop:sub-disjoint}, and \ref{prop:material-disj}.}
    \label{tab:derived-disjoint}
    \centering
    \begin{tabular}{llcc}
    \toprule
    $A$ & $B$ & Disjoint & Reason \\
    \midrule
    dbo:Chancellor & dbo:Species & $\times$ & Proposition \ref{prop:common-subclasses} \\
    dbo:LatterDaySaint & dbo:Religious & $\times$ & Proposition \ref{prop:common-subclasses} \\
    dbo:Person & dbo:Race & $\checkmark$ & Proposition \ref{prop:sub-disjoint} \\
    dbo:AcademicConference & dbo:Person & $\checkmark$ & Proposition \ref{prop:sub-disjoint} \\
    \bottomrule
    \end{tabular}
\end{table} 

Note that the list $\mathcal{L}$ assumes that the ontology designers carefully produced a taxonomy that is intended to also reflect disjointness between classes.
As shown in Section \ref{sec:resources}, however, this is not the case. The design of the taxonomy of DBpedia is structured such that disjoint axioms might result in unwanted inconsistencies. For this reason, we employ multiple metrics to evaluate the LLMs' performances, each measuring a different behavior of the model. For all metrics, a higher score indicates better performances, with $1$ being the maximum score.
In particular, disjoint recall (DR) measures how much the LLM aligns with humans by measuring the amount of \textit{true} disjointness axioms that have been identified by the LLM. This measure provides an evaluation of the reliability of the prompt. Non-disjoint F1 (NDF1) measures the F1 score between the non-disjoint couples in $L$ and the ones identified by the LLM. This provides a measure of how conservative the LLM is on its answers -- i.e. how much the LLM acknowledges the open-world assumption. The F1 metrics measure the end-to-end performances of the model. The symmetric consistency metric (SC) measures how much the answers provided by the LLM respect the symmetric property of the disjointness axiom -- i.e. if $A$ is disjoint from $B$ then $B$ is disjoint from $A$. Finally, we measure the overall accuracy of each model.


\paragraph{Prompting}
\begin{table}[ht]
    \caption{Prompting strategy templates. $A$ and $B$ are natural language labels of classes from the ontology.}
    \centering
    \begin{tabularx}{\textwidth}{l@{\hspace{0.5cm}}X}
        \toprule
        \textbf{Prompting Strategy} & \textbf{Template} \\
        \midrule
        Naive & Answer only ``yes'' or ``no''.\\
        \midrule
        Zero-shot Task Description & This is a question about ontological disjointness, answer only with ``yes'' or ``no'' \\
        \midrule
        Few-shot Task Description & This is a question about ontological disjointness, answer only with ``yes'' or ``no''. \newline Examples of disjoint are: ``person'' and ``file system'', ``tower'' and ``person'', ``place'' and ``agent'', ``continent'' and ``sea'', ``baseball league'' and ``bowling league'', ``planet'' and ``star''. \newline Examples of not disjoint are: ``basketball player'' and ``baseball player'', ``means of transportation'' and ``reptile'', ``garden'' and ``historic place'', ``president'' and ``beauty queen'', ``castle'' and ``prison''. \\
        \midrule
        \textbf{QA Strategy} & \textbf{Template} \\
        \midrule
        Positive & Is the class $A$ disjoint from $B$? \\ \midrule
        Negative & Can a $A$ be a $B$? \\ 
        \bottomrule
    \end{tabularx}
    \label{tab:prompt}
\end{table}

We adopt different prompting strategies: a naive approach, where the LLM has to autonomously understand the task, a task description approach, where the disjointedness task is described and a few-shot approach that extends the task description by also providing some positive and negative examples. For each prompt, we frame the problem as a question-answering (QA) task, where the LLM has to answer positively or negatively to classify two classes as disjoint. To identify the best QA approach, we identify two prompts: (i) the LLM has to answer positively to classify two classes as disjoint and (ii) the LLM has to answer negatively. Table \ref{tab:prompt} describes the prompt templates we used.
When possible, we rely on the instruction format of each LLM and use the \textit{Prompting Strategy} template to instruct the LLM while we use the \textit{QA Strategy} as a query to the instructed LLM.




\subsection{Experimental setup}
We perform our experiments on publicly available LLMs, to ensure full reproducibility of the experiments. For each LLM, we set the sampling temperature to $0$, to reduce the randomness of the result. Moreover, we only rely on \textit{small} LLMs -- i.e. LLMs with approximately $8$ billion of parameters. Through the use of proper optimization techniques, it is possible to run these models on consumer-level devices without the need for specialized hardware. We perform our experiments on a selection of the current state-of-the-art models, including Mistral 0.3 7B \cite{jiang2023mistral}, Gemma 2 9B~\cite{mesnard2024gemma2}, LLama 3 8B\footnote{\url{https://llama.meta.com/}}, and Qwen 2 7B \cite{yang2024qwen2}\footnote{Due to their closed-source nature and high costs, we reserve the exploration of GPT-3.5 and GPT-4 for future work.}. All experiments are run on 8-bit quantized models on an RTX3090 with 24GB of RAM. We experiment with each combination of the prompts of Table \ref{tab:prompt}.



\subsection{Results} \label{sec:results}
\begin{table}[t]
    \caption{Performance on disjointness detection for LLMs and prompt strategies. The best results for each prompt are \underline{underlined}, while the best results overall in a metric are in \textbf{bold}.}
    \centering
    \label{tab:benchmarking}
    \begin{tabular}{lllccccc}
    \toprule
    \textbf{Prompt} & \textbf{QA} & \textbf{LLM} & \textbf{DR} & \textbf{NDF1} & \textbf{F1} & \textbf{SC} & \textbf{Accuracy} \\
    \midrule
    \multirow{8}{*}{Naive} & \multirow{4}{*}{Positive} & Gemma 2 & 0.99 & 0.26 & \underline{0.53} & 0.89 & 0.42 \\
     &  & LLama 3 & 0.19 & 0.63 & 0.17 & 0.65 & 0.37 \\
     &  & Mistral 0.3 & \underline{\textbf{1.00}} & 0.03 & 0.49 & \underline{\textbf{0.98}} & 0.33 \\
     &  & Qwen 2 & 0.00 & \underline{\textbf{0.99}} & 0.01 & 0.96 & \underline{0.66} \\ \cmidrule{2-8}
    
     & \multirow{4}{*}{Negative} & Gemma 2 & 0.71 & \underline{0.91} & 0.69 & 0.85 & 0.79 \\ 
     &  & LLama 3 & 0.90 & 0.86 & \underline{0.74} & \underline{0.89}  & \underline{0.80} \\
     &  & Mistral 0.3 & 0.85 & 0.81 & 0.68 & 0.81 & 0.74 \\
     &  & Qwen 2 & \underline{0.92} & 0.80 & 0.70 & 0.84 & 0.74 \\ \midrule
    
    \multirow{8}{*}{Task description} & \multirow{4}{*}{Positive} & Gemma 2 & 0.99 & 0.35 & \underline{0.55} & 0.86 & 0.47 \\
     &  & LLama 3 & 0.86 & 0.08 & 0.45 & 0.90 & 0.31 \\
     &  & Mistral 0.3 & \underline{\textbf{1.00}} & 0.20 & 0.52 & \underline{0.91} & 0.40 \\
     &  & Qwen 2 & 0.04 & \underline{0.82} & 0.05 & 0.68 & \underline{0.49} \\ \cmidrule{2-8}
    
     & \multirow{4}{*}{Negative} & Gemma 2 & \underline{0.97} & \underline{0.83} & \underline{\textbf{0.75}} & 0.84 & \underline{0.79} \\
     &  & LLama 3 & 0.98 & 0.78 & 0.71 & \underline{0.90} & 0.75 \\
     &  & Mistral 0.3 & 0.85 & 0.76 & 0.64 & 0.76 & 0.69 \\
     &  & Qwen 2 & 0.96 & 0.61 & 0.61 & 0.72 & 0.60 \\ \midrule
    
    \multirow{8}{*}{Few shot} & \multirow{4}{*}{Positive} & Gemma 2 & 0.90 & 0.49 & \underline{0.54} & 0.83 & 0.51 \\
     &  & LLama 3 & 0.76 & 0.30 & 0.44 & 0.72 & 0.36 \\
     &  & Mistral 0.3 & \underline{0.95} & 0.20 & 0.50 & \underline{0.85} & 0.38 \\
     &  & Qwen 2 & 0.05 & \underline{0.87} & 0.07 & 0.74 & \underline{0.54} \\ \cmidrule{2-8}
    
     & \multirow{4}{*}{Negative} & Gemma 2 & 0.85 & \underline{0.89} & \underline{\textbf{0.75}} & \underline{0.81} & \underline{\textbf{0.82}} \\
     &  & LLama 3 & \underline{0.99} & 0.54 & 0.59 & 0.77 & 0.57 \\
     &  & Mistral 0.3 & 0.74 & 0.86 & 0.65 & 0.79 & 0.75 \\
     &  & Qwen 2 & 0.98 & 0.37 & 0.54 & 0.71 & 0.47 \\ \midrule
    
    \multicolumn{2}{c}{\multirow{4}{*}{Average among all prompts}} & Gemma 2 & \textbf{0.90} & 0.62 & \textbf{0.63} & \textbf{0.85} & \textbf{0.75} \\
     &  & LLama 3 & 0.78 & 0.53 & 0.52 & 0.81 & 0.47 \\
     &  & Mistral 0.3 & \textbf{0.90} & 0.48 & 0.58 & \textbf{0.85} & 0.49 \\
     &  & Qwen 2 & 0.49 & \textbf{0.74} & 0.33 & 0.78 & 0.36 \\ 
    \bottomrule
    \end{tabular}
\end{table}

The overall results are shown in Table \ref{tab:benchmarking}. In general, LLMs achieve promising results in disjointness detection. Notably, the best prompting technique is not providing few-shot examples, but rather providing the LLM with little to no description of the task. Indeed, it has been observed how few-shot prompting is more effective when in-context learning is required, while zero-shot prompting is more effective when the implicit knowledge of the LLM should be exploited \cite{reynolds2021fewshot}. 
Nonetheless, further research on few-shot prompting for disjointness classification should be performed, as lower performances can also be attributed to the amount and nature of the examples we provide in the prompt. We manually select examples that are likely to provide meaningful disjointness instances. However, a more complex approach could be employed, such as exploiting Retrieval Augmented Generation (RAG) techniques to provide examples that are more likely to be relevant for the classes used as input. Different heuristics can be used to measure the relevance of other classes, such as word embeddings or knowledge graph embeddings. 
Interestingly, framing the problem as a negative QA task -- i.e. asking whether an individual of a class can also be an instance of another class -- consistently outperforms the positive QA prompt. This could be attributed to the fact that using the negative approach is more consistent with natural language questions. LLMs can actively exploit their pre-training phase, which generally includes a fine-tuning phase to solve QA tasks akin to our negative prompt. On average, Gemma~2 performs better than the other LLMs. However, depending on the requirements, other LLMs might be better suited. For instance, Mistral 0.3 is better aligned with human judgment, since it has a higher recall on disjointness axioms.

\subsection{Disjointness on DBpedia}
Given the results of Table \ref{tab:benchmarking}, we consider Gemma 2 with task description prompt and a negative QA strategy as the most effective way of producing disjointness axioms among the methods tested. We execute Algorithm \ref{algo:query-llm} on the whole DBpedia ontology. We rely on a straightforward random selection for the pair $(\mathtt{D}_1, \mathtt{D}_2)$ (line 3).
In total, the algorithm takes $21589.75s \approx 6h$ to execute. Note that given the random selection, we are not able to exploit parallelism and query the LLM with single prompts. However, a selection strategy that enables parallel selection would greatly enhance the performance of the algorithm.
In total, we find $510,600$ disjointness axioms, which results in $\approx 98\%$ of the classes participating in at least one disjointness axiom. The number of axioms can be greatly reduced by relying on the ``pruning'' operation of Algorithm \ref{algo:query-llm} (line 15). In the case of the DBpedia ontology, the number of resulting axioms is $170,122$ -- a reduction of $\approx 66\%$.  

\begin{table}[ht]
    \caption{Example of disjointness judgments retrieved by Algorithm \ref{algo:query-llm}}
    \label{tab:DBpedia-disjoint}
    \centering
    \begin{tabular}{llc}
    \toprule
    Class $A$ & Class $B$ & Disjoint \\
    \midrule
    dbo:GeneLocation & dbo:HumanGene & $\times$ \\
    dbo:VideogamesLeague & dbo:Website & $\times$ \\
    dbo:InformationAppliance & dbo:MobilePhone & $\times$ \\
    dbo:Engineer & dbo:Embryology & $\times$ \\
    dbo:Identifier & dbo:District & $\times$ \\
    dbo:Mosque & dbo:Museum & $\checkmark$ \\
    dbo:MeanOfTransportation & dbo:Swimmer & $\checkmark$ \\
    dbo:WikimediaTemplate & dbo:WomensTennisAssociationTournament & $\checkmark$ \\
    dbo:Racecourse & dbo:Area & $\checkmark$ \\
    dbo:PlayboyPlaymate & dbo:Camera & $\checkmark$ \\
    \bottomrule
    \end{tabular}
\end{table}

For illustration and discussion purposes, Table \ref{tab:DBpedia-disjoint} shows a non-representative selection of particularly discussion-worthy positive and negative disjointness statements retrieved via Algorithm \ref{algo:query-llm}. We observe that for some class relationships, including both common-sense and domain-specific classes, our approach resulted in the ``conservative'' misclassification of classes as non-disjoint, meaning the LLM classified the classes as non-disjoint despite the classes actually being disjoint. Examples include dbo:VideogamesLeague and dbo:Website, dbo:GeneLocation and dbo:HumanGene, and dbo:Identifier and dbo:District.  
Conversely, we also observed ''aggressive''  misclassification where our approach classified classes as disjoint despite them being really non-disjoint. Straightforward examples include dbo:PlayboyPlaymate and dbo:Camera or dbo:WikimediaTemplate and dbo:WomensTennisAssociationTournament, with a more complicated example being dbo:Mosque and dbo:Museum. The latter being disproven by a counter-example, the famous Mosque Hagia Sophia in Turkey\footnote{\url{https://muze.gen.tr/muze-detay/ayasofya}}. 
To address these misclassification instances, we suspect that providing more contextual information in the prompt may improve classification accuracy, especially for domain-specific scenarios. Also, future work could be done to assess how, through prompt design, the approach could encourage more ``aggressive'' or ``conservative'' disjointness classifications in scenarios where relationships are more uncertain. 

\section{Conclusion and Future Work}
\label{sec:conclusions}

This work shows that LLMs can roughly identify and assert disjointness axioms in ontologies, with a different degree of reliability depending on the model. By harnessing their inherent background knowledge and employing strategic prompt engineering, we showed that these models can classify ontological disjointness with minimal human intervention. This capability simplifies ontology management and supports more robust reasoning in knowledge graphs. Our findings underscore the potential of LLMs as valuable tools for the automated enrichment of ontologies, which encourages future exploration and innovation in this domain.

Future works include testing the approach proposed in Section \ref{sec:approach} on other ontologies, to assess its effectiveness on different types of ontologies, including domain-specific ontologies. Additionally, comprehensive validation by human domain experts would be required to obtain conclusive insights into the degree of reliability of the axioms asserted by the LLM.

Moreover, using different LLMs with different numbers of parameters and improving and expanding our strategies for testing disjointness constitutes interesting future work. 

It could be worthwhile to look into heuristics for -- given a large list of disjointness candidate pairs -- picking those entries that are particularly ``promising''. One option would be to follow the strong disjointness assumption \cite{cornet2002usability} and pick ``sibling classes'', that is, classes $A$ and $B$ that have a common direct superclass $C$.
Furthermore, it could be interesting to test class pairs with just one or two examples of non-disjointness, as these instances may be errors to remove from the KG.
On another note, one could develop strategies for gauging the reliability of an LLM response by rephrasing the question asked. This involves adding a description of classes in prompts to see if it improves the answers, relying on proper ontology serialization techniques \cite{ringwald2024shades}. Finally, using advanced prompting techniques, such as chain-of-thought, may improve the results alongside RAG techniques to pick the few-shot examples. Similarly, a richer prompt, including more qualifying phrases such as ``at the same time'' to check the temporality of the disjointness or ``theoretically'' to force abstraction might instruct the model toward a more effective framing of the problem.

\begin{acknowledgments}
Elias Crum acknowledges funding provided by VITO NV (\verb|UG_PhD_2303_contract|).
Antonio De Santis's doctoral scholarship is funded by the Italian Ministry of University and Research (MUR) under the National Recovery and Resilience Plan (NRRP), by Thales Alenia Space, and by the European Union (EU) under the NextGenerationEU project.
Alessia Pisu acknowledges MUR and EU-FSE for financial support of the PON Research and Innovation 2014-2020 (D.M. 1061/2021). Nicolas Lazzari has received funding from the FAIR – Future Artificial Intelligence Research Foundation as part of the grant agreement MUR n. 341. Sebastian Rudolph is funded by the Bundesministerium für Bildung und Forschung (BMBF, Federal Ministry of Education and Research) and DAAD (German Academic Exchange Service) in project 57616814 (SECAI, School of Embedded and Composite AI).
\end{acknowledgments}

\clearpage
\bibliography{sample-ceur}

\appendix


\end{document}